%% file: main.tex
\documentclass{article}

\PassOptionsToPackage{authoryear,round}{natbib}
\usepackage[preprint]{neurips_2026}
\makeatletter
\renewcommand{\@noticestring}{Preprint.}
\makeatother
\usepackage[T1]{fontenc}
\usepackage[utf8]{inputenc}
\usepackage{microtype}
\usepackage{graphicx}
\usepackage{amsmath,amssymb,amsthm,mathtools}
\usepackage{amsfonts}
\usepackage{bm}
\usepackage{booktabs}
\usepackage{nicefrac}
\usepackage{enumitem}
\usepackage{xcolor}
\usepackage{hyperref}
\usepackage{cleveref}
\usepackage{array}
\usepackage{tabularx}
\usepackage{url}

\hypersetup{
  colorlinks=true,
  linkcolor=blue!50!black,
  citecolor=blue!50!black,
  urlcolor=blue!50!black,
  pdftitle={OTSS: Output-Targeted Soft Segmentation for Contextual Decision-Weight Learning},
  pdfauthor={Renjun Hu and Hyun-Soo Ahn}
}

\setlist[itemize]{leftmargin=1.4em, topsep=0.35em, itemsep=0.2em}
\setlist[enumerate]{leftmargin=1.5em, topsep=0.35em, itemsep=0.2em}

\newtheorem{theorem}{Theorem}
\newtheorem{proposition}{Proposition}
\newtheorem{corollary}{Corollary}

\newcommand{\E}{\mathbb{E}}

\newcommand{\Prob}{\mathbb{P}}
\newcommand{\argmax}{\mathop{\mathrm{arg\,max}}}
\newcommand{\argmin}{\mathop{\mathrm{arg\,min}}}
\newcommand{\tr}{\mathrm{tr}}
\newcommand{\Var}{\mathrm{Var}}
\newcommand{\Cov}{\mathrm{Cov}}

\title{OTSS: Output-Targeted Soft Segmentation for Contextual Decision-Weight Learning}
\author{%
  Renjun Hu \quad Hyun-Soo Ahn\\
  University of Michigan, Ross School of Business\\
  \texttt{renjunhu@umich.edu, hsahn@umich.edu}
}
\date{}

\begin{document}

\maketitle
\begin{abstract}
Many machine learning systems make constrained decisions by optimizing factorized objectives, but the context-specific objective is often treated as fixed. We study \emph{contextual decision-weight learning}: from logged decisions and proxy outputs, learn an optimizer-facing weight vector \(w(x)\) over interpretable decision factors \(z(x,d)\), rather than a direct policy or generic predictive score.

We propose \textbf{OTSS}, an \emph{output-targeted} soft-segmentation model that deploys the personalized decision-ready weight vector. At the function-class level, the theory highlights a hard-versus-soft distinction. Hard partitions incur an approximation-estimation tradeoff under overlap, while a realizable fixed-\(K\) soft class removes the hard-partition approximation floor and attains a parametric rate.

We evaluate OTSS in controlled benchmarks with finite evaluation libraries, where the true weight vector and downstream regret can be computed exactly. In the representative overlap setting, OTSS attains the lowest mean regret among the comparators, including EM mixture regression, the strongest soft-mixture baseline in our comparison; it matches EM on coefficient recovery while running about two orders of magnitude faster. In a matched \(K=5\) benchmark, OTSS remains competitive under hard-routed truth and improves as heterogeneity becomes softer and sample size grows. On a fixed Complete Journey retail anchor with real household covariates and action geometry, OTSS again achieves the lowest mean-regret point estimate.
\end{abstract}

\input{sections/01_introduction}

\input{sections/02_problem_setup}
\input{sections/03_otmoe}
\input{sections/04_theory_brief}

\input{sections/05_experiments_brief}

\begin{ack}
The first author thanks Wang-Chi Cheung and Chung-Piaw Teo from the National University of Singapore for early discussions on this topic; Antoine Atallah from Shopify, his former mentor at Hungryroot, for guidance during an internship that helped inspire this work; and Silviu Pitis from University of Michigan CSE for helpful feedback and comments on an earlier revision.
\end{ack}

\appendix
\input{sections/09_appendix_notes}

\bibliographystyle{plainnat}
\bibliography{references}

\end{document}

%% file: sections/01_introduction.tex
\section{Introduction}
\label{sec:intro}

Many machine learning systems make decisions by solving a constrained optimization problem over a candidate set. In such systems, the main bottleneck is often upstream rather than algorithmic: once an objective is specified, modern solvers can search large feasible sets effectively. A more challenging question is \emph{which objective should be optimized in this context?}

We study the problem, which we call \emph{contextual decision-weight learning}. For each context $x$ and candidate decision $d\in\mathcal{D}$, let $z(x,d)\in\mathbb{R}^J$ denote a vector of interpretable decision factors, and let $w^\star(x)\in\mathbb{R}^J$ denote the latent context-specific decision weight. The ideal decision solves
\[
d^\star(x)\in\arg\max_{d\in\mathcal{D}} w^\star(x)^\top z(x,d).
\]

The learning target is therefore a personalized linear objective that can be passed directly to an existing constrained optimizer. We do not observe \(w^\star(x)\) directly, but infer it from proxy outputs on logged decisions. This differs from standard decision-focused learning, inverse optimization, and contextual prediction: the goal is not to learn a policy or a generic predictive score, but to recover a context-dependent coefficient vector over the decision factors
\citep{elmachtoub2022smart, vlastelica2020blackbox, mandi2024decision, sadana2025contextual, sun2023mom, besbes2025contextual, huyuk2022icb}.

Because \(w^\star(x)\) is learned indirectly from logged outputs, fully separate models for every context are statistically infeasible, while a single pooled model assumes away heterogeneity. A natural compromise is to share statistical strength through latent groups or experts. This raises two choices: whether the groups are learned from output-relevant variation or fixed by exogenous context geometry, and whether a context is assigned to one expert or allowed to interpolate across several.

We argue that a segmentation model for this problem should be both \emph{output-targeted} and \emph{decision-ready}: routing and experts should be trained from observed proxy outputs, and the learned personalized weight vector should remain the quantity used by the optimizer. For such models, the key representational choice is whether segmentation should be hard or soft.

This distinction matters when heterogeneity is neither absent nor discrete. If \(w^\star(x)\) is nearly constant, a pooled model may suffice; if latent structure is effectively one-hot, hard routing can be competitive. Many populations lie between these extremes, with optimizer-facing trade-offs changing smoothly across contexts. In that regime, forcing each context into one segment can be too rigid in transition regions. This view is also consistent with work in choice modeling showing that observed preferences can depend on decision context rather than a single context-free utility ranking \citep{seshadri2019context}.

We propose \textbf{OTSS} (\emph{Output-Targeted Soft Segmentation}) for contextual decision-weight learning. Soft segmentation represents contexts by interpolating among expert trade-off vectors rather than committing to one regime. OTSS combines a gate with expert coefficient vectors and returns
\[
\hat w(x)=\sum_{k=1}^K \hat\alpha_k(x)\hat\beta_k.
\]
OTSS is learned from logged tuples \((x_i,d_i,y_i)\): the historical decision \(d_i\) determines the exposed factor vector \(z(x_i,d_i)\), while the observed proxy output \(y_i\) provides the supervision used to fit both the gate and the experts. Thus OTSS is trained from historical decisions paired with proxy outputs, but deployed through the learned weight map \(x\mapsto\hat w(x)\).

Section~\ref{sec:theory} studies this pooled/hard/soft comparison in a stylized family.Pooled oracle error is governed by variation in the latent mixture weight. Fixed hard classes pay a positive approximation floor under sustained overlap, while oracle-aligned hard classes improve as boundary mass vanishes. In the benchmark family studied here, the realizable soft class removes the hard-partition floor and achieves a parametric \(n^{-1}\) rate. By contrast, the balanced hard-partition benchmark remains on the nonparametric \( n^{-2/3}\) bias--variance scale, with a matching projected-overlap approximation floor at the balanced region count. This gives a quantitative prediction for the controlled experiments.
We show that weight error affects downstream decisions via the scores observed by the optimizer; the appendix records the finite-library regret-transfer bound used for exact benchmark evaluation. Together, the theory and experiments connect representational advantages in the learned objective map to measurable gains in benchmark regret.

This places the paper at the intersection of contextual and inverse optimization, inverse behavioral modeling, and latent-heterogeneity methods including mixture and routing models \citep{elmachtoub2022smart, vlastelica2020blackbox, mandi2024decision, sadana2025contextual, sun2023mom, besbes2025contextual, huyuk2022icb, seshadri2019context, mcfadden2000mixed, kamakura1989probabilistic, sen2017latent, jacobs1991adaptive, jordan1994hme, gormley2019moe, nguyen2024softmaxmoe, nguyen2024sigmoid}. Logged-feedback and slate-evaluation literatures are relevant mainly as neighboring benchmark contexts \citep{dudik2011dr, swaminathan2015crm, swaminathan2017slate, saito2020obd, wu2020mind}. Our contribution is to isolate contextual decision-weight learning as a distinct optimizer-facing problem, develop output-targeted soft segmentation as a decision-ready model class, and show through theory and controlled benchmarks why it behaves differently from pooled or hard alternatives. Concretely:
\begin{itemize}
    \item We formalize \emph{contextual decision-weight learning} for solver-based constrained optimization from observed proxy outputs on logged decisions, with pooled, hard, and soft classes as the main structural comparison.
    \item We propose \textbf{OTSS}, an output-targeted soft-segmentation model for learning optimizer-facing, decision-ready weight vectors.
    \item We show theoretically that hard partitions incur an overlap floor and, in the benchmark family studied here, the balanced hard-partition comparison has an \(O(n^{-2/3})\) upper scale with a matching approximation floor at the balanced partition size, while the realizable soft class achieves a parametric \(O(n^{-1})\) rate.
    \item We construct both controlled and retail-like benchmarks to evaluate hard versus soft segmentation under different degrees of overlap, nuisance misalignment, and relative to the ground truth.
\end{itemize}

%% file: sections/02_problem_setup.tex
\section{Problem Setup}
\label{sec:setup}
In many decision systems, the optimization layer already exists: a solver or decision engine chooses a feasible decision, each decision is summarized by interpretable factors, and the operational question is how those factors should be weighted for a given context. We study \emph{contextual decision-weight learning} in this setting. The learner does not output a black-box score or a policy directly, but a personalized coefficient vector used by the decision procedure. Each feasible decision \(d\in\mathcal{D}\) is represented by a factor vector \(z(x,d)\in\mathbb{R}^J\), and the best decision for context \(x\) is determined by an unobserved weight vector \(w^\star(x)\in\mathbb{R}^J\). For readability the main text writes a common feasible set \(\mathcal{D}\); all definitions extend to a context-dependent feasible correspondence \(\mathcal{D}(x)\), since the focus is on learning the coefficients rather than modeling the feasible set.

\subsection{Setup and logged supervision}
\label{subsec:decision_problem}
\label{subsec:logged_data}

Given a weight map \(w^\star\), the latent decision score is
\[
s^\star(x,d)=w^\star(x)^\top z(x,d),
\]
and an optimal decision can be selected as
\[
d^\star(x)\in\arg\max_{d\in\mathcal{D}} w^\star(x)^\top z(x,d).
\]
The learning target is the map \(x\mapsto w^\star(x)\). Once this map is estimated, the downstream solver can use the estimated weights in the same decision rule. Thus, the goal is to estimate optimizer-facing coefficients, rather than to learn a policy directly.

We do not observe \(w^\star(x)\) directly. Instead, we infer it from logged triples \(\mathcal{T}=\{(x_i,d_i,y_i)\}_{i=1}^n\), where \(x_i\) is the context, \(d_i\in\mathcal{D}\) is the logged decision, and \(y_i\) is an observed proxy output such as retention, conversion, or engagement. The identification challenge is that each context reveals an output only for the decision that was actually taken, while the target is a coefficient vector that can be applied to score feasible decisions. In the main text we use the binary logistic instantiation
\[
\Pr(Y=1\mid x,d)
=
\sigma\!\bigl(b^\star(x)+w^\star(x)^\top z(x,d)\bigr),
\]
where \(\sigma(t)=1/(1+e^{-t})\) and \(b^\star(x)\) is a context-only baseline term. The log provides supervision only at the realized decision: for a given context, we observe the output of the logged choice, not the outputs that would have followed from other feasible decisions. Because \(b^\star(x)\) does not vary with \(d\), it drops out of the optimizer's ranking over feasible decisions; the decision-relevant part is \(w^\star(x)^\top z(x,d)\). This logistic form is the paper's canonical working model. It matches retention- and conversion-style proxy outputs, keeps the learned coefficients interpretable, and makes explicit how logged outputs supervise the coefficient vector used for optimization. Appendix~\ref{app:output_scope} discusses extensions beyond this logistic working model.

\paragraph{Working logged-data assumptions.}
We assume a stable logging environment in which contexts are drawn from a fixed population, logged decisions are produced by the existing decision process, and proxy outputs follow the conditional response model above. The log must contain enough variation in the exposed factor vectors \(z(x_i,d_i)\) to distinguish the decision-dependent score \(w^\star(x)^\top z(x,d)\) from the context-only baseline \(b^\star(x)\). Equivalently, the local Fisher information for the action-dependent coefficients must be nonsingular. The controlled benchmarks enforce this through bounded factors and positive logged exposure over the finite evaluation library. In deployment, the log must similarly cover enough of the decisions that the solver may consider.

\subsection{Function classes and target criterion}
\label{subsec:function_classes}
\label{subsec:comparison}
\label{subsec:regret}

We compare three structural classes for the weight map \(x\mapsto w(x)\).

\paragraph{Pooled class.}
\[
\mathcal{W}_{\mathrm{pool}}
=
\left\{
w:\mathcal{X}\to\mathbb{R}^J \;:\; w(x)=\beta
\text{ for some }\beta\in\mathbb{R}^J
\right\}.
\]

\paragraph{Hard-segmented class.}
Given a hard routing map \(h:\mathcal{X}\to[K]\),
\[
\mathcal{W}_{\mathrm{hard}}(h)
=
\left\{
w:\mathcal{X}\to\mathbb{R}^J \;:\; w(x)=\beta_{h(x)}
\text{ for some }\beta_1,\ldots,\beta_K\in\mathbb{R}^J
\right\}.
\]
This class includes cluster-then-fit procedures with exogenous groups, as well as supervised one-hot segmentation when the routing map is learned from observed proxy outputs.

\paragraph{Soft-segmented class.}
\[
\mathcal{W}_{\mathrm{soft}}
=
\left\{
w:\mathcal{X}\to\mathbb{R}^J \;:\;
w(x)=\sum_{k=1}^K \alpha_k(x)\beta_k,\;
\alpha(x)\in\Delta^K,\;
\beta_k\in\mathbb{R}^J
\right\},
\]
where \(\Delta^K=\{\alpha\in\mathbb{R}_+^K:\sum_{k=1}^K\alpha_k=1\}\). This class allows contexts in transition regions to interpolate across experts rather than being forced into a single segment.

These classes isolate the representational question studied in the paper: whether heterogeneity in optimizer-facing weights should be ignored, represented by hard partitions, or represented through soft interpolation. Output-targeted training enters through how the class is fitted from logged proxy outputs, while decision readiness comes from returning the coefficient vector \(w(x)\) itself.

Given an estimate \(\hat w(x)\), the induced decision is
\[
\hat d(x)\in\arg\max_{d\in\mathcal{D}} \hat w(x)^\top z(x,d),
\]
and the downstream decision regret is
\[
R(x)
=
w^\star(x)^\top z(x,d^\star(x))
-
w^\star(x)^\top z(x,\hat d(x)).
\]
The experiments report both weight error and decision regret. The theory below studies the function-class side of this comparison: when hard segmentation is structurally too rigid, and what changes once the learner moves to a realizable soft class.

%% file: sections/03_otmoe.tex
\section{OTSS}
\label{sec:otmoe}

Section~\ref{sec:setup} defined the learning target in this paper: a context-specific coefficient vector \(w^\star(x)\) that scores feasible decisions through \(w^\star(x)^\top z(x,d)\), even though it must be inferred from observed proxy outputs on logged decisions. OTSS instantiates the soft class from Section~\ref{subsec:function_classes} for this coefficient-learning problem. It uses output-targeted training to learn both a gate over contexts and a bank of expert coefficients, then returns the interpolated weight vector \(\hat w(x)\in\mathbb{R}^J\). Thus the proxy-output model is the training interface, while the deployed quantity is the optimizer-facing coefficient vector. In the main text, we use the binary logistic working model from Section~\ref{sec:setup}, with GLM-style expert coefficients over the shared decision factors.

\begin{figure}[t]
\centering
\includegraphics[width=\linewidth]{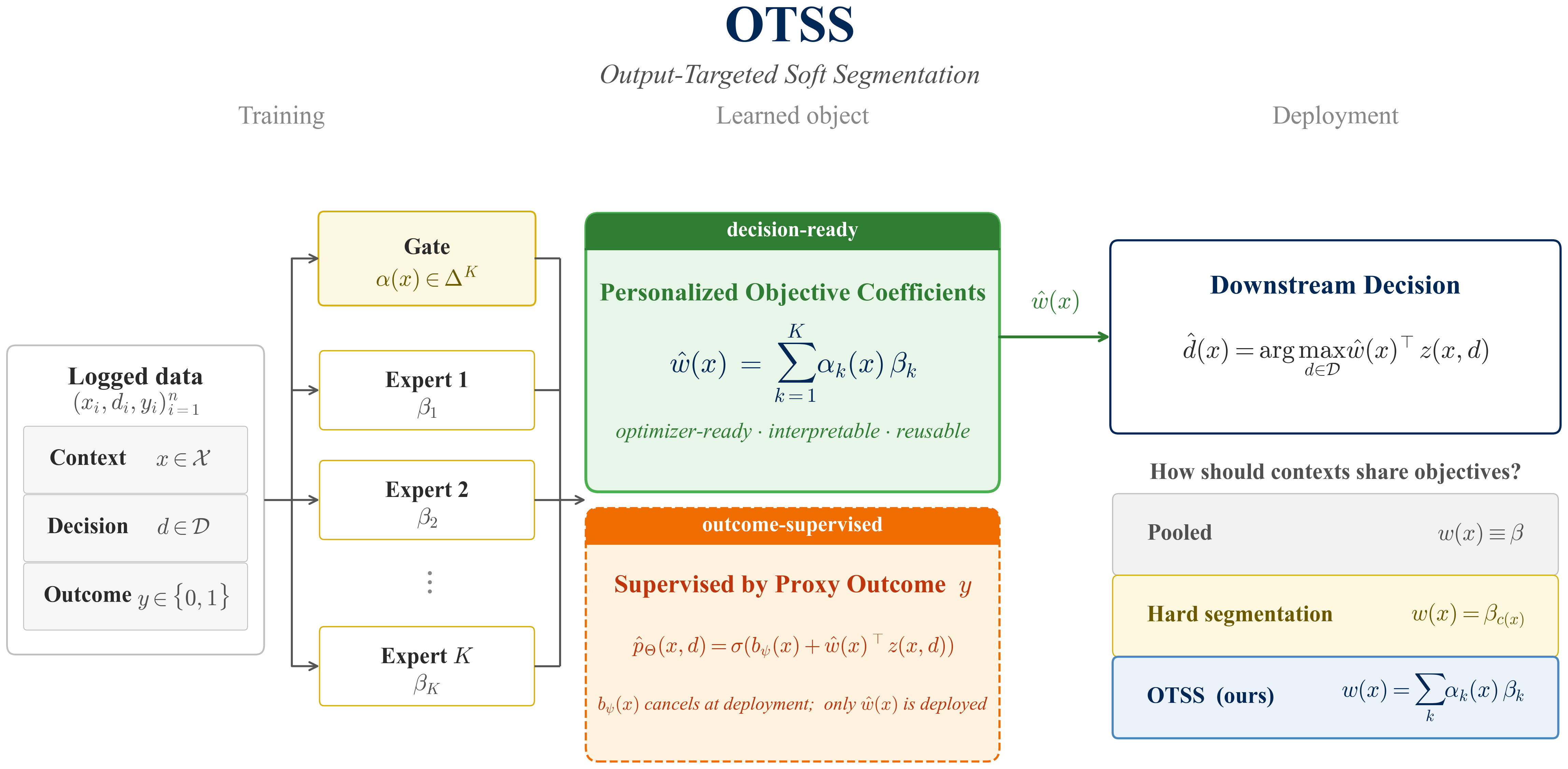}
\caption{OTSS workflow.}
\label{fig:otss_overview}
\end{figure}

\subsection{OTSS as output-targeted, decision-ready soft segmentation}
\label{subsec:otmoe_model}
\label{subsec:otmoe_output_targeted}

OTSS parameterizes the soft-segmented class with a gate over contexts and a bank of expert coefficient vectors. The gate \(G_\phi\) maps context \(x\) to mixture weights \(\alpha(x)=G_\phi(x)\in\Delta^K\), and expert \(k\) contributes \(\beta_k\in\mathbb{R}^J\). The model returns the interpolated decision-weight vector
\[
\hat w_\Theta(x)=\sum_{k=1}^K \alpha_k(x)\beta_k.
\]
A context-only baseline \(b_\psi(x)\) captures output variation common to all feasible decisions for the same context. Under the main logistic instantiation,
\[
\hat p_\Theta(x,d)
=
\sigma\!\left(
b_\psi(x)+\hat w_\Theta(x)^\top z(x,d)
\right),
\]
where \(\Theta=\{\phi,\psi,\beta_1,\ldots,\beta_K\}\). The baseline fits output levels, while the action-dependent term is the part passed to the downstream optimizer.

This is the sense in which the segmentation is output-targeted: the gate and experts are learned end-to-end from observed proxy outputs, not from an unsupervised distance in raw context space. Contexts can therefore receive similar routing weights when they imply similar trade-offs over the decision factors, even if they are not close in raw features.

\subsection{Training and decision-time prediction}
\label{subsec:otmoe_training}
\label{subsec:otmoe_deployment}
Given logged data \(\mathcal{T}=\{(x_i,d_i,y_i)\}_{i=1}^n\), OTSS is fit by minimizing the regularized logistic loss
\[
\mathcal{L}(\Theta)
=
-\frac{1}{n}\sum_{i=1}^n
\left[
y_i\log \hat p_\Theta(x_i,d_i)+(1-y_i)\log(1-\hat p_\Theta(x_i,d_i))
\right]
+
\lambda\,\mathcal{R}(\Theta).
\]
In the main controlled studies, \(G_\phi\) is a low-capacity gate over the raw context features, and the gate and experts are optimized jointly from logged proxy outputs. This keeps the main comparison focused on the pooled, hard, and soft coefficient classes rather than on routing-network capacity. Appendix~\ref{app:secondary_robustness} reports a secondary study that varies the gate family.

At decision time, the context-only baseline \(b_\psi(x)\) drops out because it does not depend on the feasible decision. OTSS therefore uses the learned weight vector to produce
\[
\hat d(x)\in\arg\max_{d\in\mathcal{D}} \hat w_\Theta(x)^\top z(x,d).
\]
The same construction nests the paper's main comparison classes: \(K=1\) gives the pooled model; a one-hot gate gives an output-targeted hard-segmentation model; and exogenously fixed routing gives a cluster-then-fit pipeline over raw context space.

Architecturally, OTSS is a mixture-of-experts model with a softmax gate and GLM experts, following the family introduced by \citet{jacobs1991adaptive} and \citet{jordan1994hme}. The contribution is to use this architecture as an optimizer-facing model of contextual decision weights, and to evaluate the resulting hard-versus-soft distinction in benchmarks that isolate overlap, nuisance misalignment, and matched hard-versus-soft truth.

%% file: sections/04_theory_brief.tex
\section{Theory}
\label{sec:theory}

The theory studies the learned weight map \(w:\mathcal{X}\to\mathbb{R}^J\), with target \(w^\star(x)\) given by the optimizer-facing decision weights from Section~\ref{sec:setup}. For a candidate class \(\mathcal{H}\), we measure its best possible approximation to the target by
\[
\inf_{w\in\mathcal{H}} \E\!\left[\|w(X)-w^\star(X)\|_2^2\right].
\]
This quantity compares which kinds of latent heterogeneity each model class can represent before the learned weights are passed to the solver. Proposition~\ref{thm:hard_partition_decomposition} gives the hard-segmentation decomposition, and Corollary~\ref{thm:quant_overlap_floor} specializes it to the two-expert overlap case. Proposition~\ref{thm:ot_decomposition} gives the soft-segmentation decomposition, and Theorem

\subsection{Hard segmentation can incur a structural approximation floor under overlap}
\label{subsec:hard_floor}

We first fix a hard segmentation and ask what error remains even before estimation. Let \(C:\mathcal{X}\to[K]\) be a measurable partition of the context space, and define
\[
\mathcal{W}_C
=
\left\{
w:\mathcal{X}\to\mathbb{R}^J
\;:\;
w(x)=\beta_{C(x)}
\text{ for some }
\beta_1,\ldots,\beta_K\in\mathbb{R}^J
\right\}.
\]
This covers cluster-then-fit and supervised one-hot routing: in either case, the learned coefficient vector is constant on each hard region. For fixed \(C\), let \(w_C^\dagger(x):=\E[w^\star(X)\mid C(X)=C(x)]\) be the conditional-mean projection of \(w^\star\) onto this class.

\begin{proposition}[Hard-partition decomposition]
\label{thm:hard_partition_decomposition}
For any hard estimator \(\hat w_C\in\mathcal{W}_C\),
\[
\E\!\left[\|w^\star(X)-\hat w_C(X)\|_2^2\right]
=
\underbrace{
\E\!\left[\|w^\star(X)-w_C^\dagger(X)\|_2^2\right]
}_{\mathcal{H}(C)}
+
\underbrace{
\E\!\left[\|w_C^\dagger(X)-\hat w_C(X)\|_2^2\right]
}_{\mathcal{E}_{\mathrm{hard}}(C)}.
\]
Equivalently,
\[
\mathcal{H}(C)
=
\E\!\left[\tr\,\Var\!\left(w^\star(X)\mid C(X)\right)\right].
\]
\end{proposition}

The term \(\mathcal{H}(C)\) is the fixed-class approximation floor: it depends on the segmentation and latent weight map, not on sample size or optimization quality. Since
\[
\inf_{w\in\mathcal{W}_C}
\E\!\left[\|w^\star(X)-w(X)\|_2^2\right]
=
\mathcal{H}(C),
\]
this floor vanishes only when \(w^\star(X)\) is constant within each hard region. If overlap creates genuine within-cell variation, additional data cannot remove it.

The decomposition is general. To make the floor explicit, specialize to the two-expert overlap family used to motivate the main benchmark. Suppose \(X\sim\mathrm{Unif}[0,1]\) and
\[
w^\star(x)
=
\alpha^\star(x)\beta_1^\star
+
\bigl(1-\alpha^\star(x)\bigr)\beta_2^\star,
\]
where \(\beta_1^\star\neq \beta_2^\star\) and \(\alpha^\star:[0,1]\to[0,1]\) is nondecreasing. Let \(\mathcal{W}^{\mathrm{hard}}_M\) denote functions \(w:[0,1]\to\mathbb{R}^J\) that are constant on each cell of a partition of \([0,1]\) into at most \(M\) intervals.Assume there exist \(I=[a,b]\subset[0,1]\) and \(\kappa>0\) such that
\[
\alpha^\star(t)-\alpha^\star(s)\ge \kappa (t-s)
\qquad\text{for all } a\le s<t\le b.
\]

\begin{corollary}[Quantitative overlap floor]
\label{thm:quant_overlap_floor}
Under the conditions above,
\[
\inf_{w\in\mathcal{W}^{\mathrm{hard}}_M}
\E\!\left[\|w^\star(X)-w(X)\|_2^2\right]
\ge
\frac{\|\beta_1^\star-\beta_2^\star\|_2^2\,\kappa^2\,(b-a)^3}{12M^2}.
\]
\end{corollary}

Thus, under sustained overlap, the best \(M\)-region hard class is limited by piecewise-constant approximation error. Increasing \(M\) reduces this hard approximation term only at the \(M^{-2}\) scale. This identifies the overlap case where hard segmentation is structurally too rigid. The two edge cases are different. Pooled fitting is accurate when the same weight vector is nearly appropriate across contexts. An oracle-aligned hard split is accurate in a different case: contexts may still require different expert weights, but most contexts are cleanly assigned to one expert, with ambiguity only near a small boundary set. This is not the same as pooling, which uses one coefficient vector for everyone. Formal pooled and aligned-hard bounds are in Appendix~\ref{app:regime_map_proofs}.

\subsection{Under realizability, OTSS has no structural approximation floor}
\label{subsec:ot_decomposition}

The soft class removes the hard-partition floor when the latent weight map is realizable as an interpolation of expert coefficients. In the benchmark family, this means
\[
w^\star(x)=\sum_{k=1}^K \alpha_k^\star(x)\beta_k^\star,
\qquad
\alpha^\star(x)\in\Delta^K,
\]
while OTSS estimates
\[
\hat w_{\mathrm{OTSS}}(x)=\sum_{k=1}^K \alpha_k(x)\beta_k.
\]

\begin{proposition}[OTSS expert/gate decomposition]
\label{thm:ot_decomposition}
For every context \(x\),
\[
\hat w_{\mathrm{OTSS}}(x)-w^\star(x)
=
\sum_{k=1}^K \alpha_k(x)\bigl(\beta_k-\beta_k^\star\bigr)
+
\sum_{k=1}^K \bigl(\alpha_k(x)-\alpha_k^\star(x)\bigr)\beta_k^\star.
\]
Moreover, if \(B_\beta:=\max_{1\le k\le K}\|\beta_k^\star\|_2\), then
\[
\E\!\left[\|\hat w_{\mathrm{OTSS}}(X)-w^\star(X)\|_2^2\right]
\le
2\underbrace{
\sum_{k=1}^K \E[\alpha_k(X)]\,\|\beta_k-\beta_k^\star\|_2^2
}_{\mathcal{E}_{\mathrm{expert}}}
+
2\underbrace{
B_\beta^2\,\E\!\left[\|\alpha(X)-\alpha^\star(X)\|_1^2\right]
}_{\mathcal{E}_{\mathrm{gate}}}.
\]
\end{proposition}

Under realizability, the proposition says something simple: if the true weights are formed by blending a fixed set of expert coefficients, then a soft model can represent that map without carving the context space into hard regions. Overlap between experts is therefore not a permanent approximation error. The remaining error comes from two estimation tasks: learning the expert coefficients and learning how much each context should use each expert. The first term in the bound tracks coefficient error; the second tracks gate error. Thus soft segmentation turns overlap from a hard-partition representation problem into an expert-and-gate estimation problem. Hard-routed truth is included as the one-hot special case \(\alpha_k^\star(x)\in\{0,1\}\), so the same class covers both one-hot structure and genuine overlap.

\subsection{Finite-sample rate separation}
\label{subsec:rate_separation}

Sections~\ref{subsec:hard_floor} and~\ref{subsec:ot_decomposition} show the mechanism: hard classes incur structural approximation cost under overlap, while realizable soft classes incur only estimation error. The theorem below compares the two rates in a fixed-\(K\) oracle-gate setting. We use \(K\) for latent soft experts and \(M\) for hard regions. The soft benchmark keeps \(K\) fixed, while \(M\) is tuned for the hard benchmark, giving the hard side its best bias--variance tradeoff.

\begin{theorem}[Hard-versus-soft rate separation]
\label{thm:rate_separation}
Let \(X\sim\mathrm{Unif}[0,1]\) and suppose
\[
w^\star(x)=\sum_{k=1}^K \alpha_k^\star(x)\beta_k^\star,
\qquad
\alpha^\star(x)\in\Delta^K,
\]
with fixed \(K\). Under the standard smoothness, overlap, and finite-dimensional MLE regularity conditions stated in Appendix~\ref{app:rate_separation_proof}, the hard and soft benchmarks have the following rates.

\paragraph{A. Hard class.}
For an \(M\)-region hard partition estimator,
\[
\E\!\left[\|\hat w_M(X)-w^\star(X)\|_2^2\right]
\le
O(M^{-2})+O(MJ/n).
\]
The two terms reflect hard-region approximation and per-region estimation. Optimizing over \(M\) gives the balanced hard rate
\[
O\!\left((J/n)^{2/3}\right).
\]
Under projected overlap, the approximation floor keeps the hard benchmark on this \(n^{-2/3}\) scale.

\paragraph{B. Soft class with oracle gate.}
In the realizable soft class with known gate \(\alpha^\star\), the oracle-gate logistic MLE satisfies
\[
\E\!\left[\|\hat w_{\mathrm{soft}}(X)-w^\star(X)\|_2^2\right]
=
O\!\left(\frac{KJ+p}{n}\right),
\]
where \(p\) counts fixed baseline or other non-expert parameters. For fixed \(K\), \(p\), and \(J\), this is the parametric \(n^{-1}\) rate.
\end{theorem}

Theorem~\ref{thm:rate_separation} isolates the source of the rate gap. A hard partition must choose how many regions to use: more regions reduce approximation error, but leave fewer samples per region. The oracle-gate soft benchmark avoids this particular tradeoff because the gate already represents the overlap, so estimation is fixed-dimensional. Appendix~\ref{app:rate_separation_proof} gives the hard-side bounds, the logistic-MLE calculation, the crossover threshold, and the corresponding local result for jointly trained gates and experts. In that local result, once the arbitrary expert labels are aligned and the estimator is in a regular neighborhood of the target solution, standard mixture-of-experts likelihood theory~\citep{jiang1999hierarchical} gives the parametric scale \(O((KJ+p_{\mathrm{loc}})/n)\), where \(p_{\mathrm{loc}}\) counts gate and baseline parameters. Appendix~\ref{app:margin_transfer_proof} then links weight error to decision regret for the finite evaluation libraries used in the experiments.

%% file: sections/05_experiments_brief.tex
\section{Controlled Experiments}
\label{sec:experiments}

The experiments ask when output-targeted, decision-ready soft segmentation improves optimizer-facing weight learning relative to pooled, exogenous, and hard-segmented alternatives. The controlled benchmarks provide the mechanism evidence; Complete Journey serves as a retail-style bridge with real covariates and action geometry. Broader sweeps and uncertainty summaries are in the appendix.

\subsection{Benchmark and representative comparison}
\label{subsec:exp_design}

We use two controlled benchmark families. The first is a two-expert setup aligned with the theory, where the true personalized weight is
\[
w^\star(x)
=
\lambda^\star(x)\beta_1^\star
+
\bigl(1-\lambda^\star(x)\bigr)\beta_2^\star,
\qquad
\lambda^\star(x)=\sigma\!\bigl(\tau\,u^\top x_{\mathrm{sig}}\bigr),
\]
so overlap is controlled directly through \(\tau\).In this benchmark, we vary sample size and nuisance variation in the raw context features. Because all methods are evaluated on the same finite decision library, both weight error and regret can be computed by enumeration. The second family is a matched hard-versus-soft benchmark: it keeps the expert bank, gate-score family, and logged-output pipeline fixed while changing only whether the latent truth is hard-routed or softly mixed:
\[
w^\star(x)=\beta_{h^\star(x)}
\qquad\text{or}\qquad
w^\star(x)=\sum_{k=1}^K \alpha_k^\star(x)\beta_k^\star,
\]
where \(h^\star:\mathcal X\to[K]\) is the hard routing map.

Across these benchmarks we compare pooled fitting, direct contextual baselines, exogenous cluster-then-fit segmentation, EM mixture regression, matched hard routing, and OTSS. The direct contextual baselines are linear, low-rank, and MLP maps from raw context to decision weights.EM mixture regression is the main classical soft-mixture comparator: it fits mixture components for the proxy-output model and combines their coefficients using posterior responsibilities. OTSS instead learns the mixture weights and expert coefficients as the decision-weight model itself. Matched hard routing is the one-hot-gate analogue of OTSS.
Detailed benchmark construction and shared training choices are deferred to Appendices~\ref{app:benchmark_summary} and~\ref{app:training_protocol}.

\Cref{tab:main_results} gives the representative numeric anchors. Panel~A is the primary overlap benchmark; Panel~B is the matched hard-versus-soft control.

\begin{table}[t]
\centering
\scriptsize
\renewcommand{\arraystretch}{0.85}
\setlength{\tabcolsep}{4pt}
\textit{Panel A: target-aligned overlap ($\tau=1.2$, nuisance scale $0.5$, eight seeds)}\\[-0.1em]
\begin{tabular*}{\linewidth}{@{\extracolsep{\fill}}lcc@{}}
\toprule
\textbf{Method} &
\textbf{regret} &
\textbf{MSE} \\
\midrule
pooled & $0.223\pm0.090$ & $0.794\pm0.015$ \\
linear contextual & $0.124\pm0.102$ & $0.554\pm0.282$ \\
low-rank contextual & $0.042\pm0.032$ & $0.209\pm0.141$ \\
MLP contextual & $0.078\pm0.043$ & $0.434\pm0.213$ \\
cluster-then-fit & $0.188\pm0.083$ & $0.685\pm0.142$ \\
EM mixture regression & $0.026\pm0.018$ & $0.138\pm0.086$ \\
matched hard routing & $0.128\pm0.121$ & $0.781\pm0.264$ \\
\textbf{OTSS} & $\mathbf{0.014\pm0.008}$ & $\mathbf{0.120\pm0.047}$ \\
\bottomrule
\end{tabular*}

\textit{Panel B: matched $K=5$ hard-versus-soft benchmark ($\mathrm{rand}=1.2$, eight seeds)}\\[-0.1em]
\begin{tabular*}{\linewidth}{@{\extracolsep{\fill}}lccc@{}}
\toprule
\textbf{Method} &
\textbf{Hard regret} &
\textbf{Soft regret} &
\textbf{High-\(n\) soft regret} \\
\midrule
pooled & $1.470\pm0.099$ & $\mathbf{0.602\pm0.101}$ & $0.567\pm0.052$ \\
linear contextual & $1.721\pm0.024$ & $0.911\pm0.022$ & $0.867\pm0.016$ \\
low-rank contextual & $1.480\pm0.078$ & $0.678\pm0.073$ & $0.634\pm0.063$ \\
cluster-then-fit & $1.497\pm0.094$ & $0.611\pm0.100$ & $0.578\pm0.058$ \\
EM mixture regression & $1.507\pm0.118$ & $0.621\pm0.022$ & $0.558\pm0.059$ \\
matched hard routing & $1.592\pm0.077$ & $0.722\pm0.084$ & $0.610\pm0.056$ \\
\textbf{OTSS} & $\mathbf{1.458\pm0.040}$ & $0.605\pm0.062$ & $\mathbf{0.553\pm0.039}$ \\
\bottomrule
\end{tabular*}
\caption{Representative benchmark anchors (mean \(\pm\) sd over eight seeds; bold is best point estimate). Panel~A reports the target-aligned overlap benchmark; Panel~B reports the matched \(K=5\) hard-versus-soft benchmark. Full paired bootstraps are in Appendix~\ref{app:bootstrap_summary}.}
\label{tab:main_results}
\end{table}

Panel~A shows the main overlap result. OTSS has the lowest mean regret, including a paired-bootstrap improvement over EM mixture regression (OTSS-minus-EM $-0.012$ $[-0.024,-0.003]$), and the best MSE point estimate. The MSE gap to EM is favorable but not bootstrap-separated. Panel~B serves as the matched control. OTSS is best under hard truth and high-\(n\) soft truth; under finite-\(n\) soft truth, it is within \(0.003\) of pooled, so we read that column as a near-tie on the regret frontier rather than a reversal.

\paragraph{Mechanism sweeps.}
\Cref{fig:main_sweeps} varies three experimental knobs: how much data are available, how sharply the latent weights switch between experts, and how much irrelevant context variation distracts clustering methods. For readability, the figure plots four structural methods; \Cref{tab:main_results} reports the full comparator set. OTSS improves steadily with more data, remains strong when the truth is softly mixed rather than cleanly one-hot, and is less distracted by nuisance context variation than cluster-then-fit. These patterns match the intended use case: soft segmentation is most useful when different contexts need different trade-offs, but the differences are not clean hard clusters.

\begin{figure}[t]
\centering
\includegraphics[width=0.72\linewidth]{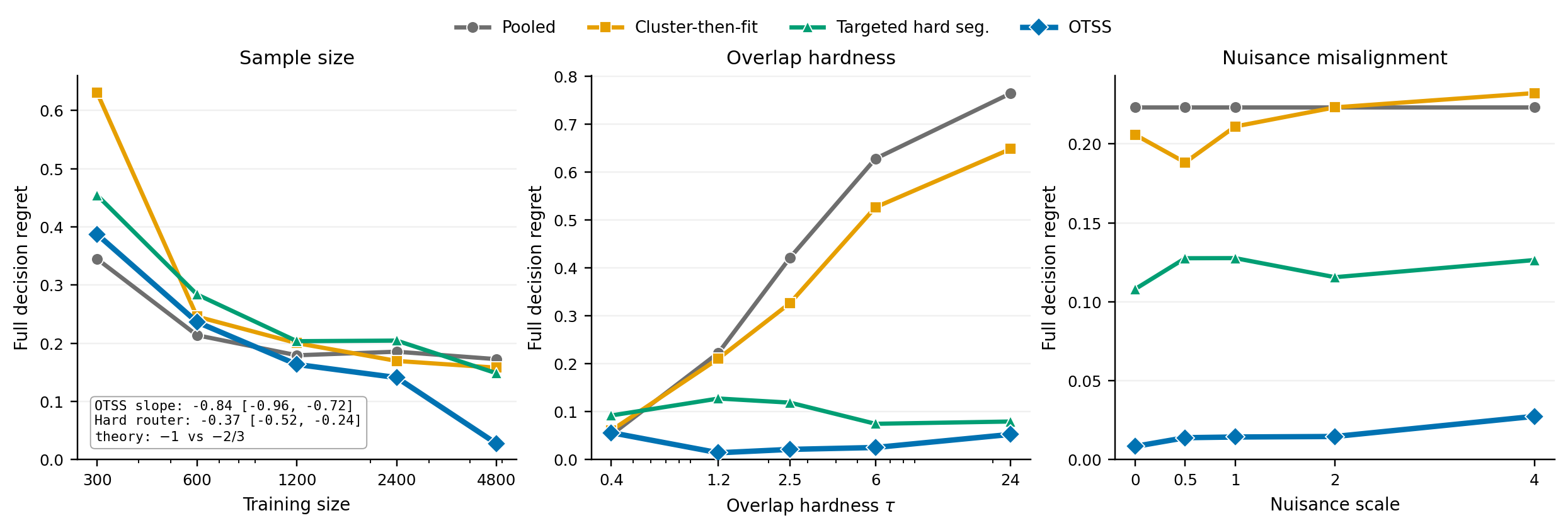}
\caption{Theorem-aligned mechanism sweeps for four structural methods (eight seeds; mean regret); \Cref{tab:main_results} reports the full comparator set.}
\label{fig:main_sweeps}
\end{figure}

EM mixture regression is the closest classical soft-mixture reference. In Panel~A, OTSS achieves lower regret than EM by paired bootstrap, while the two methods are close on coefficient MSE. In the richer \(K\!=\!5\) matched benchmark, their regret point estimates are also close. The difference is that OTSS returns the optimizer-facing coefficient vector directly and is much faster in our implementation: Appendix~\ref{app:runtime_comparison} reports roughly \(101\times\) speedup over EM in Panel~A and roughly \(230\times\) in the \(K\!=\!5\) benchmark.
\subsection{Retail bridge: Complete Journey}
\label{subsec:complete_journey}

Complete Journey tests whether the same ordering persists when contexts and feasible actions come from a public retail panel rather than a fully synthetic design. We use real household covariates and breakfast-bundle action geometry from the dunnhumby Complete Journey panel~\citep{dunnhumby_complete_journey}, while retaining a constructed latent decision-weight layer and fixed evaluation library so benchmark regret can be computed exactly.

In the main Complete Journey configuration $(\tau=1.2,\ \mathrm{rand}=0.8,\ n_{\mathrm{total}}=8{,}000,\ n_{\mathrm{train}}=5{,}000)$, OTSS has the lowest mean-regret point estimate, $0.111\pm0.041$, compared with $0.172\pm0.083$ for low-rank and $0.198\pm0.191$ for EM. The paired bootstrap gives a sharper comparison: OTSS improves over low-rank by $-0.060$ $[-0.111,-0.018]$ and over the other non-EM baselines with intervals excluding zero. The EM comparison points in the same direction, $-0.085$ $[-0.227,0.008]$, but the interval includes zero, so we treat it as directional rather than statistically separated.

%% file: sections/09_appendix_notes.tex
\section{Proofs, implementation notes, robustness studies, and scope clarifications}
\label{app:implementation_details}

This appendix collects proof details, implementation notes, benchmark-specific robustness studies, and scope clarifications that support the main text without interrupting its narrative flow. Subsections~\ref{app:quant_floor_proof}--\ref{app:margin_transfer_proof} record the theory-facing material; the remaining subsections collect benchmark construction, uncertainty summaries, robustness notes, and scope clarifications for the experimental results.

\subsection{Quantitative overlap floor proof}
\label{app:quant_floor_proof}

We first prove \Cref{thm:quant_overlap_floor}. Let
\[
\Delta:=\beta_1^\star-\beta_2^\star,
\qquad
w^\star(x)=\beta_2^\star+\alpha^\star(x)\Delta.
\]
Consider any $w\in\mathcal{W}^{\mathrm{hard}}_M$, so there exists a partition of $[0,1]$ into at most $M$ intervals and a constant vector $c_j\in\mathbb{R}^J$ on each cell $J_j$. Write
\[
c_j=\beta_2^\star+t_j\Delta+u_j,
\qquad
u_j\perp \Delta.
\]
Then
\[
\|w^\star(x)-c_j\|_2^2
=
\|(\alpha^\star(x)-t_j)\Delta-u_j\|_2^2
=
\|\Delta\|_2^2(\alpha^\star(x)-t_j)^2+\|u_j\|_2^2,
\]
so the orthogonal term can only increase the error. Therefore the best $M$-interval hard approximation lies on the line segment spanned by $\beta_1^\star$ and $\beta_2^\star$, and
\[
\inf_{w\in\mathcal{W}^{\mathrm{hard}}_M}
\E\!\left[\|w^\star(X)-w(X)\|_2^2\right]
=
\|\Delta\|_2^2
\inf_{g\in\mathcal{S}_M}
\E\!\left[(\alpha^\star(X)-g(X))^2\right],
\]
where $\mathcal{S}_M$ denotes scalar step functions on at most $M$ intervals.

Now restrict attention to the overlap interval $I=[a,b]$. Let $J\subset I$ be any subinterval of length $h$, and let $U\sim\mathrm{Unif}(J)$. Define
\[
\phi(x):=\alpha^\star(x)-\kappa x.
\]
The assumption
\[
\alpha^\star(y)-\alpha^\star(x)\ge \kappa(y-x)
\qquad\text{for }x<y\text{ in }I
\]
implies that $\phi$ is nondecreasing on $I$. Since $U$ and $\phi(U)$ are comonotone,
\[
\Cov(U,\phi(U))\ge 0.
\]
Hence
\[
\Var(\alpha^\star(U))
=
\Var(\kappa U+\phi(U))
\ge
\kappa^2\Var(U)
=
\kappa^2 h^2/12.
\]
Because the best constant fit on $J$ is the interval mean,
\[
\inf_{c\in\mathbb{R}}
\int_J (\alpha^\star(x)-c)^2\,dx
=
h\,\Var(\alpha^\star(U))
\ge
\kappa^2 h^3/12.
\]

Let the partition induced on $I$ have cells $J_1,\ldots,J_N$ with lengths $h_1,\ldots,h_N$, where $N\le M$ and $\sum_{j=1}^N h_j=b-a$. Summing the previous bound over these cells gives
\[
\int_I (\alpha^\star(x)-g(x))^2\,dx
\ge
\frac{\kappa^2}{12}\sum_{j=1}^N h_j^3.
\]
By Jensen's inequality,
\[
\sum_{j=1}^N h_j^3
\ge
N\Bigl(\frac{b-a}{N}\Bigr)^3
=
\frac{(b-a)^3}{N^2}
\ge
\frac{(b-a)^3}{M^2}.
\]
Therefore
\[
\inf_{g\in\mathcal{S}_M}
\E\!\left[(\alpha^\star(X)-g(X))^2\right]
\ge
\frac{\kappa^2(b-a)^3}{12M^2},
\]
which proves
\[
\inf_{w\in\mathcal{W}^{\mathrm{hard}}_M}
\E\!\left[\|w^\star(X)-w(X)\|_2^2\right]
\ge
\frac{\|\beta_1^\star-\beta_2^\star\|_2^2\kappa^2(b-a)^3}{12M^2}.
\]

If $\alpha^\star$ is $L$-Lipschitz, partition $[0,1]$ into $M$ equal intervals and let $g_M$ take the midpoint value of $\alpha^\star$ on each interval. On each interval of width $1/M$,
\[
|\alpha^\star(x)-g_M(x)|\le \frac{L}{2M},
\]
so
\[
\E\!\left[(\alpha^\star(X)-g_M(X))^2\right]\le \frac{L^2}{4M^2}.
\]
Multiplying by $\|\beta_1^\star-\beta_2^\star\|_2^2$ yields the stated upper bound and shows the $\Theta(M^{-2})$ rate is sharp.

\subsection{Pooled and aligned-hard regime results and proofs}
\label{app:regime_map_proofs}

\begin{proposition}[Pooled oracle error in the two-expert family]
\label{prop:pooled_oracle_two_expert}
Let
\[
\mathcal{W}^{\mathrm{pool}}
=
\{w:[0,1]\to\mathbb{R}^J : w(x)\equiv \beta \text{ for some } \beta\in\mathbb{R}^J\},
\]
and write $\Delta\beta^\star:=\beta_1^\star-\beta_2^\star$. Then
\[
\inf_{w\in\mathcal{W}^{\mathrm{pool}}}
\E\!\left[\|w^\star(X)-w(X)\|_2^2\right]
=
\|\Delta\beta^\star\|_2^2\,\Var\!\bigl(\alpha^\star(X)\bigr).
\]
\end{proposition}

\begin{proof}[Proof of \Cref{prop:pooled_oracle_two_expert}]
Write
\[
w^\star(X)=\beta_2^\star+\alpha^\star(X)\Delta\beta^\star,
\qquad
\Delta\beta^\star:=\beta_1^\star-\beta_2^\star.
\]
The best pooled coefficient is therefore
\[
\beta_{\mathrm{pool}}^\dagger
=
\E[w^\star(X)]
=
\beta_2^\star+\E[\alpha^\star(X)]\Delta\beta^\star.
\]
Hence
\[
w^\star(X)-\beta_{\mathrm{pool}}^\dagger
=
\bigl(\alpha^\star(X)-\E[\alpha^\star(X)]\bigr)\Delta\beta^\star,
\]
so
\[
\E\!\left[\|w^\star(X)-\beta_{\mathrm{pool}}^\dagger\|_2^2\right]
=
\|\Delta\beta^\star\|_2^2\,\Var\!\bigl(\alpha^\star(X)\bigr).
\]
Since $\beta_{\mathrm{pool}}^\dagger$ is the projection of $w^\star$ onto the constant class, this equals the oracle pooled approximation error.
\end{proof}

\begin{proposition}[Aligned hard upper bound near one-hot truth]
\label{prop:aligned_hard_upper}
Let $C_{1/2}(x):=\mathbf{1}\{\alpha^\star(x)\ge 1/2\}$ and let $\mathcal{W}_{C_{1/2}}$ be the associated hard class. Then
\[
\inf_{w\in\mathcal{W}_{C_{1/2}}}
\E\!\left[\|w^\star(X)-w(X)\|_2^2\right]
\le
\|\Delta\beta^\star\|_2^2\,
\E\!\left[\min\!\left\{(\alpha^\star(X))^2,\,(1-\alpha^\star(X))^2\right\}\right].
\]
Consequently, for every $\varepsilon\in(0,1/2)$,
\[
\inf_{w\in\mathcal{W}_{C_{1/2}}}
\E\!\left[\|w^\star(X)-w(X)\|_2^2\right]
\le
\|\Delta\beta^\star\|_2^2
\left(
\varepsilon^2+\frac14\,\Prob\!\bigl(\alpha^\star(X)\in[\varepsilon,1-\varepsilon]\bigr)
\right).
\]
\end{proposition}

\begin{proof}[Proof of \Cref{prop:aligned_hard_upper}]
For the oracle-aligned hard class used in Proposition~\ref{prop:aligned_hard_upper}, consider the predictor
\[
\tilde w(x)
:=
\beta_1^\star \mathbf{1}\{\alpha^\star(x)\ge 1/2\}
+
\beta_2^\star \mathbf{1}\{\alpha^\star(x)<1/2\},
\]
which belongs to $\mathcal{W}_{C_{1/2}}$. If $\alpha^\star(x)\ge 1/2$, then
\[
w^\star(x)-\tilde w(x)
=
\bigl(1-\alpha^\star(x)\bigr)(\beta_2^\star-\beta_1^\star),
\]
while if $\alpha^\star(x)<1/2$, then
\[
w^\star(x)-\tilde w(x)
=
\alpha^\star(x)(\beta_1^\star-\beta_2^\star).
\]
Therefore, pointwise,
\[
\|w^\star(x)-\tilde w(x)\|_2^2
=
\|\Delta\beta^\star\|_2^2
\min\!\left\{(\alpha^\star(x))^2,\,(1-\alpha^\star(x))^2\right\}.
\]
Taking expectations gives
\[
\begin{aligned}
\inf_{w\in\mathcal{W}_{C_{1/2}}}
\E\!\left[\|w^\star(X)-w(X)\|_2^2\right]
&\le
\E\!\left[\|w^\star(X)-\tilde w(X)\|_2^2\right] \\
&\le
\|\Delta\beta^\star\|_2^2
\E\!\left[\min\!\left\{(\alpha^\star(X))^2,\,(1-\alpha^\star(X))^2\right\}\right].
\end{aligned}
\]
For the second claim, note that for any $a\in[0,1]$ and any $\varepsilon\in(0,1/2)$,
\[
\min\{a^2,(1-a)^2\}
\le
\varepsilon^2+\frac14\,\mathbf{1}\{a\in[\varepsilon,1-\varepsilon]\}.
\]
Applying this with $a=\alpha^\star(X)$ and taking expectations yields
\[
\E\!\left[\min\!\left\{(\alpha^\star(X))^2,\,(1-\alpha^\star(X))^2\right\}\right]
\le
\varepsilon^2+\frac14\,\Prob\!\bigl(\alpha^\star(X)\in[\varepsilon,1-\varepsilon]\bigr),
\]
which proves the stated upper bound.
\end{proof}

\subsection{Rate separation proof}
\label{app:rate_separation_proof}

\begin{theorem}[Local joint estimation under regular soft truth]
\label{thm:local_joint_rate}
Consider the OTSS logistic model from Section~\ref{sec:otmoe}. Let
\[
\Theta=(\phi,\psi,\beta_1,\ldots,\beta_K),
\qquad
p_{\mathrm{loc}}:=\dim(\phi,\psi),
\]
and suppose the logged data are generated by some
\[
\Theta^\star=(\phi^\star,\psi^\star,\beta_1^\star,\ldots,\beta_K^\star)
\]
such that
\[
\Prob(Y=1\mid x,d)
=
\sigma\!\left(
b_{\psi^\star}(x)+\sum_{k=1}^K \alpha_k(x;\phi^\star)\,{\beta_k^\star}^\top z(x,d)
\right).
\]
Assume:
\begin{enumerate}[label=(\roman*)]
\item \textbf{Finite-dimensional regular soft model.}
The maps $\phi\mapsto \alpha(x;\phi)$ and $\psi\mapsto b_\psi(x)$ are twice continuously differentiable on an open neighborhood of $(\phi^\star,\psi^\star)$.

\item \textbf{Bounded design and smooth score.}
The decision factors $z(x,d)$ are almost surely bounded, and the first and second derivatives of the conditional log-likelihood with respect to $\Theta$ are dominated by an integrable envelope on a neighborhood of $\Theta^\star$.

\item \textbf{Local identification up to label permutation.}
There exists a neighborhood $U$ of the equivalence class of $\Theta^\star$ such that, after relabeling experts and gate outputs by the best permutation, the population logistic risk has quadratic growth:
\[
\mathcal{L}_{\rm pop}(\widetilde\Theta)-\mathcal{L}_{\rm pop}(\Theta^\star)
\ge
c_0\|\widetilde\Theta-\Theta^\star\|_2^2
\qquad\text{for all }\Theta\in U,
\]
for some $c_0>0$, where $\widetilde\Theta$ denotes the relabeling of $\Theta$ minimizing Euclidean distance to $\Theta^\star$.

\item \textbf{Regular local curvature.}
The population Hessian of $\mathcal{L}_{\rm pop}$ at $\Theta^\star$ is finite and nonsingular on the identified tangent directions.

\item \textbf{Localized estimator.}
Let $\hat\Theta$ be a local minimizer of the regularized empirical loss $\mathcal{L}(\Theta)$ from Section~\ref{sec:otmoe}, with a twice differentiable penalty and regularization level $\lambda=O(n^{-1})$, and suppose $\hat\Theta\in U$.
\end{enumerate}
Then, after optimal label alignment,
\[
\E\!\left[
\|\widetilde\phi-\phi^\star\|_2^2
+
\|\widetilde\psi-\psi^\star\|_2^2
+
\sum_{k=1}^K \|\widetilde\beta_k-\beta_k^\star\|_2^2
\;\middle|\;
\hat\Theta\in U
\right]
\le
C\,\frac{KJ+p_{\mathrm{loc}}}{n},
\]
for a constant $C$ depending only on the local smoothness and curvature constants.
\end{theorem}

\begin{corollary}[Weight error for the joint estimator]
\label{cor:local_joint_weight_rate}
Under the assumptions of Theorem~\ref{thm:local_joint_rate}, suppose further that there exists $L_\alpha<\infty$ such that
\[
\sup_{x\in\mathcal X}
\|\alpha(x;\phi)-\alpha(x;\phi^\star)\|_1
\le
L_\alpha \|\phi-\phi^\star\|_2
\qquad\text{for all }\phi\in U,
\]
and define
\[
B_\beta:=\max_{1\le k\le K}\|\beta_k^\star\|_2.
\]
Then the induced OTSS weight map
\[
\hat w(x)=\sum_{k=1}^K \alpha_k(x;\hat\phi)\hat\beta_k
\]
satisfies
\[
\E\!\left[
\|\hat w(X)-w^\star(X)\|_2^2
\;\middle|\;
\hat\Theta\in U
\right]
\le
C'\,\frac{KJ+p_{\mathrm{loc}}}{n}.
\]
If the optimization and initialization scheme lands in \(U\) with high probability, the conditional local rate describes the regular-case behavior of the jointly trained estimator.
\end{corollary}

\paragraph{Proof sketch and scope note for the local joint result.}
Theorem~\ref{thm:local_joint_rate} is included as the standard localized finite-dimensional $M$-estimation consequence that matches the main text's local-basin qualifier; we do not claim it as a separate global convergence theorem. Write the regularized empirical objective as
\[
\mathcal{L}_n(\Theta)+\lambda P(\Theta),
\]
where $P$ is twice differentiable near $\Theta^\star$ and $\lambda=O(n^{-1})$. After optimal label alignment, local optimality of $\hat\Theta\in U$ gives the first-order condition
\[
\nabla \mathcal{L}_n(\hat\Theta)+\lambda \nabla P(\hat\Theta)=0
\]
on the identified tangent directions. A Taylor expansion around $\Theta^\star$ yields
\[
0=\nabla \mathcal{L}_n(\Theta^\star)+\nabla^2 \mathcal{L}_{\rm pop}(\bar\Theta)\,(\widetilde\Theta-\Theta^\star)+r_n+\lambda \nabla P(\hat\Theta),
\]
for some $\bar\Theta$ on the segment joining $\Theta^\star$ and $\widetilde\Theta$, where $r_n$ is the empirical Hessian fluctuation term.

Under assumptions (i)--(iv), the score at $\Theta^\star$ has mean zero and Euclidean norm \(O_p(\sqrt{(KJ+p_{\mathrm{loc}})/n})\), the empirical Hessian is uniformly close to the population Hessian on $U$, and the quadratic-growth / nonsingularity assumptions imply that the identified Hessian is invertible with bounded inverse after relabeling. The penalty contributes only \(O(n^{-1})\) at the score scale. Solving the linearized optimality equation therefore gives
\[
\|\widetilde\Theta-\Theta^\star\|_2=O_p\!\left(\sqrt{\frac{KJ+p_{\mathrm{loc}}}{n}}\right),
\]
and hence the conditional squared-error bound
\[
\E\!\left[\|\widetilde\Theta-\Theta^\star\|_2^2\mid \hat\Theta\in U\right]
=
O\!\left(\frac{KJ+p_{\mathrm{loc}}}{n}\right).
\]
This is the quantity summarized in Theorem~\ref{thm:local_joint_rate}.
Specializing this local result to the fixed-\(K\) oracle-gate family used in \Cref{thm:rate_separation} gives the \(O((KJ+p_{\mathrm{loc}})/n)\) bridge quoted in the main text.

For Corollary~\ref{cor:local_joint_weight_rate}, write
\[
\hat w(x)-w^\star(x)
=
\sum_{k=1}^K \alpha_k(x;\hat\phi)\bigl(\hat\beta_k-\beta_k^\star\bigr)
+
\sum_{k=1}^K \bigl(\alpha_k(x;\hat\phi)-\alpha_k(x;\phi^\star)\bigr)\beta_k^\star.
\]
The first term is controlled by the simplex constraint on \(\alpha(x;\hat\phi)\), and the second by the Lipschitz gate condition:
\[
\left\|
\sum_{k=1}^K \bigl(\alpha_k(x;\hat\phi)-\alpha_k(x;\phi^\star)\bigr)\beta_k^\star
\right\|_2
\le
B_\beta\,\|\alpha(x;\hat\phi)-\alpha(x;\phi^\star)\|_1
\le
B_\beta L_\alpha \|\hat\phi-\phi^\star\|_2.
\]
Squaring, integrating over \(X\), and applying Theorem~\ref{thm:local_joint_rate} yields the stated \(O((KJ+p_{\mathrm{loc}})/n)\) bound for the joint estimator's weight error.

\begin{proof}[Proof of \Cref{thm:rate_separation}]
We prove each part in turn.

\paragraph{Part (i): Hard rate.}
Fix a regular \(M\)-bin partition \(\mathcal{P}_M=\{I_1,\ldots,I_M\}\) with cell probabilities \(p_m=\Prob(X\in I_m)\asymp 1/M\). Assume throughout this hard-side comparison that \(n p_m\) is large relative to \(J\) uniformly in \(m\), and that the within-cell Fisher information at the pseudo-truths defined below is uniformly nonsingular. For each cell, let
\[
\beta_m^\circ
\in
\argmin_{\beta\in\mathbb{R}^J}
\E\!\left[
\ell\!\left(Y,b(X)+\beta^\top z(X,D)\right)
\;\middle|\;
X\in I_m
\right],
\]
where \(\ell\) is the Bernoulli negative log-likelihood and the same finite-dimensional baseline specification is held fixed. Define the corresponding population hard logistic pseudo-truth \(w_M^\circ(x)=\beta_m^\circ\) for \(x\in I_m\), and assume the hard-side pseudo-truth compatibility condition
\[
\E\!\left[\|w_M^\circ(X)-w^\star(X)\|_2^2\right]
\le
\frac{C_{\mathrm{app}}}{M^2}.
\]

We first control the approximation side. The projected-overlap lower bound follows by reducing any vector-valued hard approximation to a scalar step-function approximation. For any hard \(w\in\mathcal{W}^{\mathrm{hard}}_M\), define \(h(x):=v^\top w(x)\). Then \(h\) is an \(M\)-interval scalar step function, and
\[
\E\|w(X)-w^\star(X)\|_2^2
\ge
\E\!\left[(h(X)-g(X))^2\right],
\qquad
g(x):=v^\top w^\star(x).
\]
Since \(g\) satisfies the same monotone-variation condition used in \Cref{thm:quant_overlap_floor},
\[
\inf_{w\in\mathcal{W}^{\mathrm{hard}}_M}
\E\!\left[\|w^\star(X)-w(X)\|_2^2\right]
\ge
\frac{\kappa^2(b-a)^3}{12M^2}
\]
for every \(M\)-interval hard class. The Lipschitz assumption on \(w^\star\) gives the matching approximation upper scale: partition \([0,1]\) into \(M\) equal intervals and approximate \(w^\star\) by the cell midpoint or cell-average value. The squared error on each cell is \(O(M^{-2})\), hence the integrated approximation error is \(O(M^{-2})\).

This is the hard-side approximation scale used below; the compatibility condition ensures that within-cell logistic misspecification does not introduce an additional floor beyond \(O(M^{-2})\).

The empirical hard estimator fits a logistic regression in each cell and returns \(\hat w_M(x)=\hat\beta_m\) for \(x\in I_m\). Under bounded features and uniform within-cell Fisher curvature at the pseudo-true targets,
\[
\lambda_{\min}\!\left(I_m(\beta_m^\circ)\right)\ge c_F>0
\qquad\text{for all }m,
\]
the classical local logistic-MLE rate gives, on the regular count event,
\[
\E\!\left[\|\hat\beta_m-\beta_m^\circ\|_2^2\right]
\le
\frac{c_1 J}{N_m\vee1}
\le
\frac{c_2 J}{n p_m},
\]
where \(c_1,c_2\) depend on the bounded-design and curvature constants but not on \(n\), \(M\), or \(m\). Aggregating over cells:
\[
\mathcal{E}_{\mathrm{hard}}
=
\sum_{m=1}^M p_m\,\E\!\left[\|\hat\beta_m-\beta_m^\circ\|_2^2\right]
\le
\sum_{m=1}^M p_m \cdot \frac{c_2 J}{np_m}
=
\frac{c_2 MJ}{n}.
\]
For this specific within-cell MLE construction, the total hard MSE is therefore bounded by the pseudo-truth approximation and estimation terms:
\[
\E\!\left[\|\hat w_M(X)-w^\star(X)\|_2^2\right]
\;\le\;
\frac{2C_{\mathrm{app}}}{M^2}+\frac{2c_2 MJ}{n},
\]
where constants are absorbed in the theorem statement. Separately, every $M$-region hard class retains the projected-overlap approximation floor
\[
\inf_{w\in\mathcal{W}^{\mathrm{hard}}_M}
\E\!\left[\|w(X)-w^\star(X)\|_2^2\right]
\ge
\frac{C_{\mathrm{ov}}}{M^2},
\qquad
C_{\mathrm{ov}}:=\frac{\kappa^2(b-a)^3}{12}.
\]
Minimizing the upper bound over $M$:
\[
\frac{d}{dM}\left(\frac{C_{\mathrm{app}}}{M^2}+\frac{c_2 J}{n}M\right)
=
-\frac{2C_{\mathrm{app}}}{M^3}+\frac{c_2 J}{n}
=0
\quad\Longrightarrow\quad
M^\star
=
\left(\frac{2C_{\mathrm{app}}\,n}{c_2 J}\right)^{\!1/3}.
\]
We choose the nearest admissible integer to \(M^\star\), assuming the balanced choice remains in the regular-bin regime \(n/M^\star\) large relative to \(J\).
Substituting back:
\[
\frac{C_{\mathrm{app}}}{(M^\star)^2}+\frac{c_2 J}{n}M^\star
=
\Theta\!\left(C_{\mathrm{app}}^{1/3}\left(\frac{J}{n}\right)^{\!2/3}\right).
\]
This gives the advertised balanced hard upper scale \(O((J/n)^{2/3})\). The approximation floor evaluated at the same balanced choice $M^\star$ has the matching functional form
\[
\frac{C_{\mathrm{ov}}}{(M^\star)^2}
=
\Omega\!\left(\Bigl(\frac{J}{n}\Bigr)^{\!2/3}\right),
\]
with a constant depending on \(C_{\mathrm{ov}}\), \(C_{\mathrm{app}}\), and \(c_2\). This is the sense in which the hard side of \Cref{thm:rate_separation} remains nonparametric under overlap.

\paragraph{Part (ii): Soft rate.}
Under realizability with known gate $\alpha^\star$, the model is
\[
\Prob(Y=1\mid x,d)
=
\sigma\!\left(b(x)+\sum_{k=1}^K \alpha_k^\star(x)\,\beta_k^{\star\top} z(x,d)\right).
\]
Define the augmented feature vector
\[
\Phi(x,d):=\bigl(\varphi(x),\alpha_1^\star(x)z(x,d),\ldots,\alpha_K^\star(x)z(x,d)\bigr)\in\mathbb{R}^{p+KJ},
\]
where \(\varphi(x)\) denotes the fixed finite-dimensional basis used by the baseline term \(b_\psi(x)\), and stack the finite-dimensional parameter as \(\theta:=(\psi,\beta_1,\ldots,\beta_K)\in\mathbb{R}^{p+KJ}\). Then the linear predictor is \(\theta^\top \Phi(x,d)\), so the oracle-gate model is a standard logistic regression in \(p+KJ\) parameters.
Here \(p\) counts fixed non-expert parameters in the oracle-gate regression; \(p_{\mathrm{loc}}\) above counts gate and baseline parameters in the local joint-estimation statement.

Under the stated regularity---bounded augmented features, say
\[
\|\Phi(X,D)\|_\infty \le B_\Phi
\qquad \text{a.s.},
\]
and Fisher information bounded below in a neighborhood of truth, in particular \(\lambda_{\min}(I(\theta^\star))\ge c_F>0\)---the oracle-gate MLE satisfies
\[
\E\!\left[\|\hat\theta^{\mathrm{MLE}}-\theta^\star\|_2^2\right]
\le
C\,\frac{KJ+p}{n},
\]
for a constant \(C\) depending only on the bounded-design and curvature constants. Since the expert part contributes \(KJ\) coordinates, the weight MSE transfers directly:
\begin{align*}
\E\!\left[\|\hat w_{\mathrm{soft}}(X)-w^\star(X)\|_2^2\right]
&=
\E\!\left[\left\|\sum_{k=1}^K \alpha_k^\star(X)(\hat\beta_k-\beta_k^\star)\right\|_2^2\right] \\
&\le
\E\!\left[\|\hat\theta^{\mathrm{MLE}}-\theta^\star\|_2^2\right]
=
O\!\left(\frac{KJ+p}{n}\right),
\end{align*}
where the inequality uses $\|\alpha^\star(x)\|_2\le\|\alpha^\star(x)\|_1=1$.

\paragraph{Crossover.}
The balanced hard upper scale is
\[
C_h\Bigl(\frac{J}{n}\Bigr)^{2/3},
\]
and the oracle-gate soft scale is
\[
C_s\,\frac{KJ+p}{n},
\]
for constants \(C_h,C_s\) that absorb overlap, smoothness, and curvature. The soft side is smaller once
\[
n
\gtrsim
\left(
\frac{C_s(KJ+p)}{C_h J^{2/3}}
\right)^3.
\]
For fixed \(K\), fixed \(p\), and fixed constants, the oracle-gate soft side is therefore parametric and eventually dominates the balanced hard \(n^{-2/3}\) scale.
\end{proof}

\subsection{Margin-local regret-transfer proof}
\label{app:margin_transfer_proof}

\begin{proposition}[Margin-local decision stability]
\label{prop:margin_local_regret}
For the finite-library regret-transfer bound used in controlled benchmark evaluation, assume $\mathcal{D}$ is finite. For each context $x$, define the true and estimated decision scores
\[
s^\star_x(d):=w^\star(x)^\top z(x,d),
\qquad
\hat s_x(d):=\hat w(x)^\top z(x,d),
\]
let
\[
d^\star(x)\in\argmax_{d\in\mathcal{D}} s^\star_x(d),
\qquad
\hat d(x)\in\argmax_{d\in\mathcal{D}} \hat s_x(d),
\]
and define the oracle margin
\[
\gamma(x)
:=
s^\star_x(d^\star(x))
-
\max_{d\in\mathcal{D}\setminus\{d^\star(x)\}} s^\star_x(d),
\]
with the convention $\gamma(x)=0$ when the oracle optimum is not unique. Also define the score perturbation level
\[
\delta(x)
:=
\max_{d\in\mathcal{D}}
\left|
\bigl(\hat w(x)-w^\star(x)\bigr)^\top z(x,d)
\right|.
\]
Then the downstream regret from Section~\ref{subsec:regret} satisfies
\[
R(x)
\le
2\,\delta(x)\,\mathbf{1}\!\left\{\gamma(x)\le 2\,\delta(x)\right\}.
\]
In particular, whenever $\gamma(x)>2\,\delta(x)$, the estimated decision is unchanged:
\[
\hat d(x)=d^\star(x),
\qquad
R(x)=0.
\]
\end{proposition}

\begin{proof}[Proof of \Cref{prop:margin_local_regret}]
Fix a context $x$ and abbreviate
\[
d^\star:=d^\star(x),
\qquad
\hat d:=\hat d(x),
\qquad
e_x(d):=\hat s_x(d)-s^\star_x(d).
\]
By definition of $\delta(x)$,
\[
|e_x(d)|\le \delta(x)
\qquad\text{for all } d\in\mathcal{D}.
\]
Using the optimality of $\hat d$ for the estimated score,
\[
\hat s_x(\hat d)\ge \hat s_x(d^\star),
\]
we obtain
\begin{align*}
R(x)
&=
s^\star_x(d^\star)-s^\star_x(\hat d) \\
&=
\bigl(s^\star_x(d^\star)-\hat s_x(d^\star)\bigr)
+
\bigl(\hat s_x(d^\star)-\hat s_x(\hat d)\bigr)
+
\bigl(\hat s_x(\hat d)-s^\star_x(\hat d)\bigr) \\
&\le
|e_x(d^\star)|+0+|e_x(\hat d)| \\
&\le
2\,\delta(x).
\end{align*}
This is the global perturbation bound.

To prove decision stability away from low-margin contexts, suppose $\gamma(x)>2\delta(x)$. Then for every $d\neq d^\star$,
\begin{align*}
\hat s_x(d^\star)-\hat s_x(d)
&=
\bigl(s^\star_x(d^\star)-s^\star_x(d)\bigr)
+
\bigl(\hat s_x(d^\star)-s^\star_x(d^\star)\bigr)
-
\bigl(\hat s_x(d)-s^\star_x(d)\bigr) \\
&\ge
\gamma(x)-|e_x(d^\star)|-|e_x(d)| \\
&\ge
\gamma(x)-2\delta(x) \\
&>0.
\end{align*}
Hence $\hat s_x(d^\star)>\hat s_x(d)$ for every $d\neq d^\star$, so $\hat d=d^\star$ and therefore $R(x)=0$.

Combining the two cases gives
\[
R(x)\le 2\,\delta(x)\,\mathbf 1\!\left\{\gamma(x)\le 2\,\delta(x)\right\},
\]
which proves the proposition.
\end{proof}

\begin{corollary}[Localized expected regret under a margin condition]
\label{cor:localized_regret_transfer}
Under the notation of Proposition~\ref{prop:margin_local_regret}, for every threshold $\eta>0$,
\[
\E[R(X)]
\le
2\eta\,\Prob\!\bigl(\gamma(X)\le 2\eta\bigr)
+
\frac{2\,\E[\delta(X)^2]}{\eta}.
\]
If, in addition,
\[
\sup_{x\in\mathcal{X},\,d\in\mathcal{D}}\|z(x,d)\|_2\le B_z
\]
and the oracle margin obeys the lower-tail condition
\[
\Prob\!\bigl(\gamma(X)\le t\bigr)\le C_{\mathrm{mar}}\,t^\alpha
\qquad\text{for all } t>0
\]
for some constants $C_{\mathrm{mar}},\alpha>0$, then
\[
\E[R(X)]
\le
C\!\left(B_z,C_{\mathrm{mar}},\alpha\right)
\left(
\E\!\left[\|\hat w(X)-w^\star(X)\|_2^2\right]
\right)^{\frac{1+\alpha}{2+\alpha}}.
\]
\end{corollary}

\begin{proof}[Proof of \Cref{cor:localized_regret_transfer}]
Fix $\eta>0$. By \Cref{prop:margin_local_regret},
\[
R(X)\le 2\,\delta(X)\,\mathbf 1\!\left\{\gamma(X)\le 2\,\delta(X)\right\}.
\]
On the event $\{\delta(X)\le \eta\}$, the implication
\[
\gamma(X)\le 2\,\delta(X)\quad\Longrightarrow\quad \gamma(X)\le 2\eta
\]
holds, so
\[
2\,\delta(X)\,\mathbf 1\!\left\{\gamma(X)\le 2\,\delta(X),\,\delta(X)\le \eta\right\}
\le
2\eta\,\mathbf 1\!\left\{\gamma(X)\le 2\eta\right\}.
\]
On the complementary event $\{\delta(X)>\eta\}$, we simply keep the global factor $2\delta(X)$. Therefore
\[
R(X)
\le
2\eta\,\mathbf 1\!\left\{\gamma(X)\le 2\eta\right\}
+
2\,\delta(X)\,\mathbf 1\!\left\{\delta(X)>\eta\right\}.
\]
Taking expectations yields
\[
\E[R(X)]
\le
2\eta\,\Prob\!\bigl(\gamma(X)\le 2\eta\bigr)
+
2\,\E\!\left[\delta(X)\mathbf 1\!\left\{\delta(X)>\eta\right\}\right].
\]
Finally, since $\delta\,\mathbf 1\{\delta>\eta\}\le \delta^2/\eta$ pointwise,
\[
\E\!\left[\delta(X)\mathbf 1\!\left\{\delta(X)>\eta\right\}\right]
\le
\frac{\E[\delta(X)^2]}{\eta},
\]
which proves the first claim.

For the second claim, assume
\[
\sup_{x\in\mathcal X,\,d\in\mathcal D}\|z(x,d)\|_2\le B_z
\qquad\text{and}\qquad
\Prob\!\bigl(\gamma(X)\le t\bigr)\le C_{\mathrm{mar}}\,t^\alpha
\quad\text{for all } t>0.
\]
Then
\[
\delta(x)
=
\max_{d\in\mathcal D}
\left|
\bigl(\hat w(x)-w^\star(x)\bigr)^\top z(x,d)
\right|
\le
B_z\,\|\hat w(x)-w^\star(x)\|_2,
\]
so
\[
\E[\delta(X)^2]
\le
B_z^2\,\E\!\left[\|\hat w(X)-w^\star(X)\|_2^2\right].
\]
Using the margin-tail bound in the first term gives
\[
\E[R(X)]
\le
2^{1+\alpha}C_{\mathrm{mar}}\,\eta^{1+\alpha}
+
\frac{2\,B_z^2\,\E\!\left[\|\hat w(X)-w^\star(X)\|_2^2\right]}{\eta}.
\]
Write
\[
V:=B_z^2\,\E\!\left[\|\hat w(X)-w^\star(X)\|_2^2\right]
\]
and
\[
a:=2^{1+\alpha}C_{\mathrm{mar}}.
\]
Then
\[
\E[R(X)]\le a\,\eta^{1+\alpha}+\frac{2V}{\eta}.
\]
Optimizing the right-hand side over $\eta>0$ gives
\[
\eta^\star
=
\left(\frac{2V}{(1+\alpha)a}\right)^{\!1/(2+\alpha)}.
\]
Substituting $\eta^\star$ back into the bound yields
\[
\E[R(X)]
\le
C\!\left(B_z,C_{\mathrm{mar}},\alpha\right)\,
\left(
\E\!\left[\|\hat w(X)-w^\star(X)\|_2^2\right]
\right)^{\frac{1+\alpha}{2+\alpha}},
\]
where the constant absorbs the factor $B_z^{2(1+\alpha)/(2+\alpha)}$. This proves the corollary.
\end{proof}

\subsection{Benchmark summary}
\label{app:benchmark_summary}

Table~\ref{tab:benchmark_summary} summarizes each benchmark's truth structure, experimental axes, and the scientific question it is meant to answer.

\begin{table}[ht]
\centering
\footnotesize
\setlength{\tabcolsep}{4pt}
\begin{tabularx}{\linewidth}{@{}>{\raggedright\arraybackslash}p{0.24\linewidth}
                >{\raggedright\arraybackslash}X
                >{\raggedright\arraybackslash}X
                >{\raggedright\arraybackslash}X@{}}
\toprule
\textbf{Benchmark family} & \textbf{Truth structure} & \textbf{Main experimental axes} & \textbf{Question addressed} \\
\midrule
Theorem-aligned overlap benchmark &
Two-expert latent weight with controlled overlap and nuisance geometry &
sample size, overlap, misalignment &
Cleanest theory-facing hard-vs-soft test \\
\addlinespace[0.3em]
Matched hard-versus-soft benchmark &
Same expert bank and latent score family under hard-routed or soft-routed truth &
hard/soft truth, randomness, sample size, expert count &
Matched adaptivity check under hard and soft truth \\
\addlinespace[0.3em]
Complete Journey benchmark &
Real household histories and exact breakfast templates with synthetic latent weights &
canonical anchor, seed variation, overlap over retained template library &
Public retail bridge under real geometry \\
\addlinespace[0.3em]
Many-expert robustness study &
Sparse experts, sparse gate dependence, locally sparse expert activation, calibrated overlap &
expert count, sample size, calibrated overlap, over-specified $K$ &
Richer-latent-structure robustness check \\
\bottomrule
\end{tabularx}
\caption{Benchmark construction, axes, and questions addressed.}
\label{tab:benchmark_summary}
\end{table}

\subsection{Shared training and evaluation}
\label{app:training_protocol}

Several design choices are held fixed across methods so that differences are attributable to function class rather than to training convenience.

\paragraph{Shared logged proxy-output supervision.}
All fitted methods are trained on the same logged tuples $(x_i,d_i,y_i)$ and use the same proxy-output modeling objective within a benchmark family. In particular, the direct contextual baselines, the supervised hard-route model, and OTSS all receive the same logged signal rather than separate labels or auxiliary supervision.

\paragraph{Shared validation philosophy.}
For the controlled synthetic benchmarks, each run uses a validation split carved from the training portion, with validation size
\[
n_{\mathrm{val}}=\max\{100,\lfloor 0.2\,n_{\mathrm{train}}\rfloor\}.
\]
Hyperparameters and stopping decisions are chosen by validation performance rather than by test regret. In the learning-curve experiments, the test set is held fixed while the training prefix grows, so changes with sample size are not confounded with changes in evaluation examples.

\paragraph{Matched model selection.}
For pooled and cluster-based logistic fits, regularization and latent component count are chosen over small validation grids; the low-rank contextual baseline likewise tunes over a small rank grid. The direct contextual baselines, the matched hard-routing model, and OTSS share the same train/validation split, validation-loss early stopping, and restart-and-select philosophy, with the same small validation-grid and restart-budget protocol across runs.

\paragraph{Overlap calibration.}
Whenever the comparison spans multiple expert counts or multiple randomness levels under soft truth, overlap is calibrated by a target top gate probability rather than by reusing a single nominal temperature. This is important in both the many-expert robustness study and the matched hard-versus-soft benchmark, because otherwise changes in $K$ would mechanically change route sharpness.

\paragraph{Replication budgets.}
Replication budgets vary with benchmark cost, but both representative table panels, the main synthetic mechanism sweeps, and the displayed Complete Journey canonical anchor use eight seeds; heavier supplementary sweeps may use smaller budgets and always report the budget used. Panel~A reports the target-aligned overlap anchor $(\tau=1.2,\ \mathrm{nuisance\ scale}=0.5)$ used as the representative theorem-aligned setting in the main text. For Complete Journey, the canonical main-text anchor is $(\tau=1.2,\ \mathrm{rand}=0.8,\ n_{\mathrm{total}}=8{,}000,\ n_{\mathrm{train}}=5{,}000)$.

\paragraph{Complete Journey data use.}
The Complete Journey retail bridge uses the provider-distributed dunnhumby source-file dataset~\citep{dunnhumby_complete_journey}, which contains household-level transactions, customer attributes for selected households, and direct-marketing contact history. We use the data under the provider's access terms, retain derived benchmark covariates and breakfast-bundle templates, and report only aggregate benchmark results.

\paragraph{Benchmark evaluation.}
All primary controlled-regret numbers use exact search over the finite benchmark library, keeping the evaluation focused on coefficient learning rather than library-construction or solver artifacts.

\subsection{Primary metrics and diagnostics}
\label{app:metric_summary}

The paper reports a small set of metrics repeatedly because each serves a different role.

\paragraph{Weight MSE.}
In controlled settings, weight MSE is
\[
\mathbb{E}\bigl[\|\hat w(X)-w^\star(X)\|_2^2\bigr].
\]
It is the most direct measure of recovery of the latent optimizer-facing coefficient vector and is the metric most directly connected to the theory.

\paragraph{Benchmark decision regret.}
This is the primary downstream metric in the controlled benchmarks:
\[
\mathbb{E}\Bigl[\max_{d\in\mathcal{D}} w^\star(X)^\top z(X,d)\;-\;w^\star(X)^\top z\bigl(X,\hat d(X)\bigr)\Bigr],
\qquad
\hat d(X)\in\arg\max_{d\in\mathcal{D}} \hat w(X)^\top z(X,d).
\]
It measures the decision loss induced by replacing $w^\star(x)$ with $\hat w(x)$ on the benchmark feasible set. In general deployment, the same learned weight vector is passed to the downstream solver over its feasible set; in the controlled benchmarks, the finite library makes this regret exactly computable.

\paragraph{Auxiliary diagnostics.}
Held-out log-loss, match rate, gate entropy, cluster alignment, and effective expert usage are diagnostic only. They help explain \emph{why} a method succeeds or fails, but exactly computed regret over the benchmark library remains the central evaluation target.

\subsection{Bootstrap uncertainty for headline comparisons}
\label{app:bootstrap_summary}

We use a paired bootstrap over the eight seed-level run summaries (5{,}000 resamples of seeds with replacement) for the headline comparisons in the main text.

\paragraph{Panel A.}
The OTSS-minus-baseline regret differences are: pooled $-0.21$ $[-0.26,-0.15]$, cluster-then-fit $-0.17$ $[-0.23,-0.12]$, linear contextual $-0.11$ $[-0.18,-0.05]$, low-rank contextual $-0.028$ $[-0.049,-0.007]$, MLP contextual $-0.064$ $[-0.094,-0.039]$, matched hard routing $-0.11$ $[-0.20,-0.04]$, and EM mixture regression $-0.012$ $[-0.024,-0.003]$. The paired intervals exclude zero for every Panel~A regret comparison, including EM mixture regression. On coefficient MSE, the OTSS-minus-EM difference is $-0.018$ $[-0.075,+0.028]$: the point estimate favors OTSS, while the paired interval supports no decisive MSE separation.

\paragraph{Panel B.}
Paired intervals exclude zero for OTSS improvements over the matched hard router (hard truth $-0.13$ $[-0.20,-0.07]$, soft truth $-0.12$ $[-0.15,-0.08]$, soft high-$n$ $-0.07$ $[-0.14,-0.01]$) and over the linear contextual baseline ($-0.26$, $-0.31$, and $-0.25$ across the three columns, all with CIs excluding zero), while the gaps to pooled, cluster-then-fit, and EM mixture regression are numerically small at these anchors.

\paragraph{Complete Journey.}
At the canonical retail anchor, OTSS has the lowest listed mean-regret point estimate: OTSS-minus-low-rank is $-0.060$ $[-0.111,-0.018]$, OTSS-minus-EM is $-0.085$ $[-0.227,+0.008]$, and paired intervals exclude zero for OTSS improvements over pooled, cluster, hard-routed, linear, and gradient-boosting (GBM) baselines.

\subsection{Controlled benchmark feasible sets and exact optimization}
\label{app:feasible_sets}

In the controlled synthetic benchmarks, the evaluation feasible set is an explicit finite library rather than an implicit combinatorial object. This is an experimental design choice, not a restriction on the general formulation: outside these benchmarks, \(\mathcal{D}(x)\) may be any solver-accessible feasible set. For notational simplicity, the main text writes a common set $\mathcal{D}$ rather than a context-dependent feasible correspondence. In the experiments, a shared benchmark library is scored for every context, and regret is computed by exact enumeration over that library. The downstream problem takes the form
\[
\arg\max_{d\in\mathcal{D}} w^\top z(x,d),
\]
where, for evaluation only, $\mathcal{D}=\{d_1,\ldots,d_M\}$ is a finite collection of benchmark decisions and the factor map $z(x,d)$ is evaluated exactly for each one. Consequently, both the oracle decision
\[
d^\star(x)\in\arg\max_{d\in\mathcal{D}} w^\star(x)^\top z(x,d)
\]
and the learned decision
\[
\hat d(x)\in\arg\max_{d\in\mathcal{D}} \hat w(x)^\top z(x,d)
\]
are computed by exact enumeration over the benchmark library.

This design is intentional: the paper studies the upstream problem of learning optimizer-facing coefficients, not the downstream problem of designing a new combinatorial solver. The many-expert robustness study and Complete Journey follow the same principle, using retained evaluation libraries that preserve exact regret evaluation.

\subsection{Hard-versus-soft comparison scope}
\label{app:comparison_scope}

The empirical comparison separates three questions: whether generic contextualization helps beyond pooling, whether output-targeted routing helps beyond exogenous hard partitioning, and whether soft interpolation helps once a latent representation has been learned. Accordingly, the paper reports pooled, cluster-then-fit, and supervised hard-routing comparators rather than collapsing hard alternatives into one baseline. Pooled locates the nearly constant regime; cluster-then-fit represents the exogenous hard-partition alternative from \Cref{thm:hard_partition_decomposition}; and the supervised hard gate isolates the incremental value of soft interpolation.

\subsection{Matched hard-versus-soft truth benchmark}
\label{app:schema_fairness_scope}

The matched hard-versus-soft benchmark is constructed so that the hard-truth and soft-truth cases differ only in the final composition rule, not in the latent ingredients. Both cases use the same expert bank $\{\beta_k^\star\}_{k=1}^K$, the same latent score family, the same logged proxy-output training pipeline, and the same regret evaluation over the benchmark library. The hard-truth case sets
\[
w^\star(x)=\beta_{h^\star(x)},
\]
while the soft-truth case sets
\[
w^\star(x)=\sum_{k=1}^K \alpha_k^\star(x)\beta_k^\star.
\]
This makes the comparison substantially cleaner than testing OTSS only under a soft ground truth. The benchmark also includes a compact randomness axis \((0.0, 0.4, 1.2)\) that perturbs latent score geometry, and the soft-truth case calibrates overlap to a common target top-probability. The default benchmark is therefore centered on \emph{decision-relevant score misalignment}, not synthetic Gaussian-mixture recovery.

\subsection{Bernoulli-logistic output family}
\label{app:output_scope}
\label{app:outcome_scope}

The theory and controlled experiments in this paper are developed only for the Bernoulli-logistic case, which is both practically natural for retention- and conversion-style proxy outputs and analytically transparent. Other GLM-style output families fit the same coefficient-learning template when the decision-dependent term remains \(w(x)^\top z(x,d)\) before the response link, but they would require separate benchmark design and analysis and are therefore out of scope here.

\subsection{Additional robustness notes}
\label{app:secondary_robustness}

\paragraph{Expert count robustness details.}
\label{app:k_robustness_details}
The over-specified-$K$ experiments are practical robustness checks. With true $K=2$ and fitted $K=8$, the effective number of active experts, measured by $\exp(H(\alpha))$, is approximately~3; for true $K\in\{3,5\}$, the degradation from $K=8$ is about $0.01$ or less in absolute regret. The evidence is meant to show robustness to moderate over-specification, not exact expert-number recovery.

\paragraph{Gate misspecification.}
When the true routing is nonlinear, a simple linear gate is misspecified. In the robustness study, the linear gate is best under linear truth (regret $0.019$ versus $0.037$ for a polynomial gate). Under quadratic truth, the misspecified linear gate remains close to cluster-then-fit ($0.101$ versus $0.108$), while the flexible gate improves to $0.050$ and lowers weight MSE from $0.480$ to $0.320$. This supports that the hard-versus-soft comparison is not an artifact of one gate parameterization.

\paragraph{Representative runtime comparison.}
\label{app:runtime_comparison}
We compare wall-clock cost for EM mixture regression against OTSS on the three representative anchors, using the same local CPU-only software environment and eight seeds. The measured environment was an Apple M1 Max with 10 CPU cores and 32GB RAM, Python 3.11.9, NumPy 2.4.4, pandas 3.0.2, and scikit-learn 1.8.0; no GPU was used. Both methods include model selection: EM searches over $K$ and ridge strength $C$, while OTSS searches over restart seeds. Absolute times and ratios are hardware and implementation dependent; in this measured CPU setup, OTSS is roughly two orders of magnitude faster in the two-expert overlap benchmark and roughly $230\times$ faster in the richer matched $K\!=\!5$ benchmark.

\begin{table}[ht]
\centering
\scriptsize
\setlength{\tabcolsep}{4pt}
\begin{tabular*}{\linewidth}{@{\extracolsep{\fill}}lccc@{}}
\toprule
\textbf{Benchmark} & \textbf{EM mix reg.} & \textbf{OTSS} & \textbf{EM / OTSS} \\
\midrule
Panel A target-aligned overlap & $17.73\pm0.90$s & $0.17\pm0.05$s & $101.4\times$ \\
Panel B hard truth & $128.82\pm9.22$s & $0.56\pm0.04$s & $228.6\times$ \\
Panel B soft truth & $128.55\pm9.37$s & $0.56\pm0.06$s & $229.6\times$ \\
\bottomrule
\end{tabular*}
\caption{Representative wall-clock comparison on the same local environment, averaged over eight seeds.}
\label{tab:runtime_comparison}
\end{table}